\pdfoutput=1 
\documentclass[11pt]{article}

\usepackage[font=libertinus, citestyle=numeric]{kurbanlab}
\usepackage[flushleft]{threeparttable}
\usetikzlibrary{patterns}

\usepackage{tikz}
\usetikzlibrary{arrows.meta}

\DeclareAffiliation{hbku}{%
  College of Science and Engineering, Hamad Bin Khalifa University, Doha, Qatar}

\DeclareAffiliation{tamu}{%
  Department of Computer and Electrical Engineering,
  Texas A\&M University, College Station, TX, USA}

\DeclareAffiliation{iub}{%
  Luddy School of Informatics, Computing, and Engineering,
  Indiana University Bloomington, Bloomington, IN, USA}

\newcommand{\ecs}{\mathrm{ECS}}
\newcommand{\rea}{\mathrm{REA}}
\newcommand{\Cent}{\mathcal{C}}
\newcommand{\Und}{\mathcal{U}}
\newcommand{\dprob}{\pi}
\newlength{\cilen}
\newcommand{\ci}[3]{%
  \begin{tikzpicture}[baseline=-0.55ex,x=\cilen,y=\cilen]
    \draw[gray!25,line width=2.6pt] (0,0) -- (1,0);
    \draw[black!70,line width=1.1pt] (#2,0) -- (#3,0);
    \fill[black] (#1,0) circle (1.15pt);
  \end{tikzpicture}}

\title{Consistency Has a Computable Blind Spot: A Commutation Theory of Label-Free Reliability for Vision-Language Figure Reading}
\RunningTitle{Consistency Has a Computable Blind Spot}

\Author{Rasul Khanbayov}{hbku}
\Author{Hasan Kurban}{hbku}

\CodeURL{https://github.com/KurbanIntelligenceLab/rend-equiv}
\Venue{Preprint}

\begin{document}
\maketitle

\begin{abstract}
Label-free reliability for vision-language models rests on \emph{invariance}: perturb the input and a faithful reader's answer should not change. This has a known blind spot, a systematic misreading survives the perturbation and gets certified wrong, which we show is computable, not just real: an error is invisible to an edit exactly when the two commute, so the errors a suite cannot reach form its joint centralizer, a set that shrinks as edits are added and can be written down rather than guessed at. We act on the complementary relation, \emph{equivariance}: edit a figure's data and the correct answer must change by a computable amount. Two matched edits are provably complete for affine reading errors; no suite of swap edits is complete for label permutations, and cyclic relabeling closes most of that gap. We instantiate the theory as the Equivariance-Consistency Score, a label-free, training-free detector, and release \textbf{REND-EQUIV}, pairing matched invariance and equivariance sets over identical data. The predicted ordering holds across three models and a hand-labeled population immune to the one circularity in how it is selected; a second invariance-family method confirms the blind spot belongs to the relation, not to any implementation; and cyclic relabeling delivers its predicted gain on a matched real sample. The same characterization explains a reported inversion of this ordering in the classifier metamorphic-testing literature: detectability is a joint property of the relation and the fault class, never of the relation alone.
\end{abstract}

\printkeywords

\section{Introduction}
\label{sec:intro}

Reliability estimation for vision-language models (VLMs) on quantitative figures increasingly rests on \emph{consistency}: a model that answers stably is treated as more trustworthy than one that does not \citep{khan2024consistency,decc2024,zoomconsistency2026}. A figure is rendered from data by a program, so the same data can be redrawn under cosmetic styles that share an exact answer, and a faithful reader should be invariant to the style. Recent work turns this into a label-free reliability signal \citep{renderequiv_prior}; hallucination detectors perturb the visual and textual prompt similarly \citep{vluncertainty2024}, as do metamorphic robustness frameworks \citep{metara2026}. All are the same relation wearing different clothes: perturb the input, expect the answer to hold.

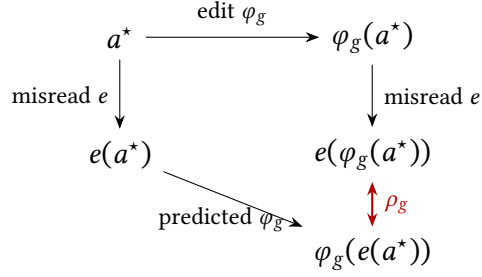
\begin{figure}[t]
\centering
\begin{tikzpicture}[
  font=\small,
  >={Stealth[length=1.8mm]},
  scale=1.15,
  every node/.style={transform shape}
]
\node (a)  at (0.15,1.5)  {$a^\star$};
\node (pa) at (3.05,1.5)  {$\varphi_g(a^\star)$};
\node (ea) at (0.15,0.15) {$e(a^\star)$};
\node (ep) at (3.05,0.15) {$e(\varphi_g(a^\star))$};
\node (pe) at (3.05,-1.0) {$\varphi_g(e(a^\star))$};

\draw[->] (a) -- 
  node[above,font=\scriptsize]{edit $\varphi_g$} (pa);

\draw[->] (a) -- 
  node[left,font=\scriptsize]{misread $e$} (ea);

\draw[->] (pa) -- 
  node[right,font=\scriptsize]{misread $e$} (ep);

\draw[->] (ea) -- 
  node[below,font=\scriptsize,pos=0.42]{predicted $\varphi_g$} (pe);

\draw[<->,red!70!black,thick] (ep) -- 
  node[right,font=\scriptsize,red!70!black] {$\rho_g$} (pe);

\end{tikzpicture}

\caption{\textbf{A systematic misread is invisible exactly when this square closes.}
Reading down then across applies the misread $e$ and then the predicted answer transformation; reading across then down applies the edit and then the same misread. The residual $\rho_g$ is the gap between the two endpoints, vanishing exactly when $e$ commutes with $\varphi_g$ (Lemma~\ref{prop:commute}).}
\label{fig:teaser}
\end{figure}

The blind spot is structural, not incidental. A model that misreads a figure in a style-independent way, adding a fixed offset on a logarithmic axis, scaling every reported value by a constant, is unaffected by restyling, so its answers agree across re-renderings and an invariance signal reads that agreement as confidence. The metamorphic-testing literature names this failure: semantic-consistency relations produce false negatives when a model hallucinates the same incorrect assertion across all input variations \citep{mtsurvey2026}. The limitation belongs to the relation, not to any implementation of it, so every method in this family inherits it, whether it perturbs the prompt \citep{vluncertainty2024} or the rendering \citep{renderequiv_prior}.

The complementary relation is equivariance. Edit the figure's underlying \emph{data} and the correct answer must change by a computable amount: scaling a series by $c$ scales a read-off by $c$; deleting the tallest bar sends ``which is largest'' to the runner-up. We control the edit, so the required answer-delta is known without annotation and the check is label-free (Figure~\ref{fig:teaser}).

The question this paper answers is not whether equivariance testing helps, but \emph{which errors it can and cannot reach, and how to design an edit suite that closes the gap}, edit suites are currently assembled by intuition, with no way to answer that question in advance. The answer is algebraic: an error is invisible to an edit exactly when it commutes with that edit's answer-transform, so the errors a suite misses form the joint centralizer of its transforms. Within the affine error class, two edits are provably \emph{complete}; within label permutations no suite of swaps is complete, and we quantify what escapes and how to shrink it, a design rule that predicts in advance which edit types will and will not fire, and which we test.

\paragraph{Contributions.}

\noindent\textbf{(1) An algebraic characterization of detectability.} Write $\varphi_g$ for the answer-transform an edit $g$ induces on the correct answer. The errors invisible to a suite $G$ are exactly those commuting with every $\varphi_g$, the joint centralizer of the transforms, and this set shrinks monotonically as edits are added (Theorem~\ref{thm:centralizer}). Invariance is the case $\varphi_g=\mathrm{id}$, which everything commutes with, so the blind spot is a property of the relation, not of any implementation: it applies equally to re-rendering \citep{renderequiv_prior}, prompt perturbation \citep{vluncertainty2024}, and metamorphic robustness suites \citep{metara2026}. Metamorphic testing studies completeness combinatorially, where it is Set-Cover-hard \citep{minmrcomplete2026}; algebraic structure buys a closed form instead.

\noindent\textbf{(2) A matched completeness and incompleteness pair, confirmed empirically.} A two-element suite $\{\times c,\ +\delta\}$ with $c\neq 1$, $\delta\neq 0$ detects every non-identity affine error, and neither edit alone suffices (Theorem~\ref{thm:affine}). No suite of swaps is complete for label permutations: a swap leaves $2(m{-}2)!$ of $m!$ errors undetectable, and cyclic relabeling reduces this to $m$ whenever $m\ge5$ (Proposition~\ref{prop:perm}, Corollary~\ref{cor:design}), a benchmark-design rule we test directly: on a matched sample of real chart instances at $m{=}6$, replacing swap edits with cyclic relabeling raises label-error detection eightfold, from $5.3\%$ to $44.0\%$, matching the closed-form prediction almost exactly (Section~\ref{sec:theory}).

\noindent\textbf{(3) ECS and REND-EQUIV.} The Equivariance-Consistency Score instantiates the theory as a label-free, training-free detector, and REND-EQUIV pairs matched invariance and equivariance sets over identical data (Sections~\ref{sec:method} and~\ref{sec:benchmark}).

\noindent\textbf{(4) Confirmation of the predicted ordering, on real errors and independent of circularity.} The per-edit-type profile follows the theory across three models, including its predicted \emph{failures} (Section~\ref{sec:experiments}), and survives a hand-labeled check that removes the one circularity in how the tested population is selected. A second, independent invariance-family baseline confirms the blind spot belongs to the relation, not to any one implementation, and the equivariance signal reaches more of the invariance-blind error mass than any baseline that varies the model's output instead of the input. The same theory accounts for a reported discrepancy with the metamorphic-testing literature on classifiers, where the ordering over these two relation families is inverted (Section~\ref{sec:discussion}), the theory is no longer resting on one model, one population, or one baseline family.

Controlled data edits are not new: they have served as training supervision \citep{chartcf2026} and counterfactual evaluation \citep{cvqa2023,visualcounterfact2025}. What is new is the characterization of which errors such edits reach, and the design rule that follows.

\section{Related work}
\label{sec:related}

\paragraph{Invariance-based reliability for figures.} The closest prior work uses semantics-preserving re-rendering as a label-free invariance signal, both to predict correctness and to aggregate answers \citep{renderequiv_prior}; chart-robustness studies perturb style and report instability similarly \citep{robustcqa2024,losingplot2025,chartrvr2025,robustcot2026}. All are invariance relations; we take that line as our point of departure and characterize the failure mode it names but does not resolve.

\paragraph{Metamorphic testing.} Metamorphic testing detects faults through input-output relations when an oracle is unavailable \citep{mtsurvey2026,mtllm2025}. MetaRA applies such relations to multimodal VQA robustness \citep{metara2026}, predominantly invariance-style, targeting robustness assessment rather than a reliability detector with an exact oracle. \citet{mtsurvey2026} identify the limitation we target, semantic-consistency relations yield false negatives against consistent hallucinations, and survey executable-oracle change-relations in text-to-SQL. Completeness there is addressed combinatorially: \citet{minmrcomplete2026} show selecting a complete MR subset is Set-Cover-equivalent (NP-hard), and \citet{mtadequacy2024} give adequacy criteria over enumerated mutants. We answer the same completeness question differently for our setting, a closed form rather than a combinatorial search (Contribution 1).

\paragraph{Controlled edits in vision-language evaluation.} Edited inputs widely test whether a model tracks visual evidence rather than priors: human-edited control groups \citep{hallusionbench2024}, object swaps for pixel-level grounding \citep{hallusegbench2025}, descriptive-versus-reasoning chart hallucination splits \citep{charthal2025}, code-controlled chart edits as training supervision \citep{chartcf2026}, and counterfactual VQA \citep{cvqa2023,visualcounterfact2025}. All score the edited input against a label, they know the \emph{new} correct answer. An equivariance relation instead supplies a \emph{map} $\varphi_g$ from the old answer to the new one: knowing the map supports detection, since the residual is computable against the model's own base answer with no label at all, and this same map is what the commutation theory acts on.

\paragraph{Consistency as a confidence signal.} Response resampling \citep{khan2024consistency,wang2023selfconsistency}, semantics-preserving prompt perturbation \citep{vluncertainty2024}, task decomposition \citep{decc2024}, multi-step geometry \citep{zoomconsistency2026}, and trained calibration \citep{vlcalibration2026} all estimate confidence from model behavior, varying the model's output or prompt, or requiring training, but never varying the data against a known answer-transform. Selective prediction and calibration supply our evaluation framing \citep{elyaniv2010foundations,guo2017calibration,srinivasan2024selective}; test-time scaling motivates the inference cost of running edits \citep{snell2024scaling,chartcoca2025}.

\section{Method}
\label{sec:method}

\paragraph{Edits and answer-transforms.} Let $D$ be the data underlying a figure and $q$ a quantitative question with exact answer $a^\star(D,q)$ computed by a deterministic program. An \emph{edit operator} $g$ transforms the data and is paired with a known \emph{answer-transform} $\varphi_g$ with
\begin{equation}
a^\star\big(g(D),q\big) = \varphi_g\big(a^\star(D,q)\big).
\label{eq:equiv}
\end{equation}
Edits are realized by re-rendering $g(D)$ in the \emph{same} cosmetic style as $D$, isolating the data change from the style change.

\paragraph{Equivariance-Consistency Score.} Let $\hat y = f(x,q)$ be the model's answer on the base rendering $x=R(D;\theta)$ and $\hat y_g$ its answer on the edited figure. A faithful reader satisfies $\hat y_g = \varphi_g(\hat y)$. Define the per-edit residual $\rho_g = d(\hat y_g,\ \varphi_g(\hat y))$, with $d$ relative error for continuous answers and $0/1$ disagreement for categorical and ordinal ones,
\begin{equation}
\ecs(D,q) = \frac{1}{|G|}\sum_{g\in G}\mathbf{1}\big[\rho_g \le \tau_g\big] \in [0,1].
\label{eq:ecs}
\end{equation}
$\tau_g$ is $0$ for categorical and ordinal answers, where the transform is exact, and a $5\%$ relative tolerance for continuous answers, held constant across every operator and chart family. ECS references the model's own base answer rather than $a^\star$, so it is label-free. We flag an answer unreliable when $\ecs < t$, and also study $\min(\rea,\ecs)$ where $\rea$ is an invariance score \citep{renderequiv_prior}: since $\min(\rea,\ecs)\le\rea$ pointwise, at a fixed threshold the combination flags a superset of what invariance flags, converting missed errors into caught ones and correct answers into false alarms, in a ratio the data must settle.

\paragraph{Composition of $G$ and scope.} Each instance carries one identity (zoom) edit as an invariance control alongside its equivariance edits, entering the average in Eq.~\ref{eq:ecs}. By Lemma~\ref{prop:commute} it can never produce a residual for a style-invariant error, so including it lowers sensitivity by a known amount and reported rates are conservative. Not every edit applies to every question type (Section~\ref{sec:benchmark}), so $|G|$ varies per instance: $|G|\in\{1,2,3\}$, mean $2.61$ across the released dataset. ECS needs the figure's data, so it applies to programmatically renderable figures, and detects rather than repairs.

\section{What equivariance testing can detect}
\label{sec:theory}

Model a reader as producing, for the quantity asked, a deterministic answer map: on data $D$ it reports $\hat a = e(a^\star(D,q))$, so $e=\mathrm{id}$ is a perfect reader. An error is \emph{systematic} when $e$ does not depend on the cosmetic style $\theta$. Throughout, $\mathcal{E}$ is a class of candidate error maps. For an edit $g$ we write $\tau$ for its tolerance $\tau_g$ from Eq.~\ref{eq:ecs} whenever the edit is clear from context, and for $\tau\ge 0$ define
\[
\Cent_\tau(\varphi)=\{e\in\mathcal{E} :\ d\big(e(\varphi(a^\star)),\varphi(e(a^\star))\big)\le\tau\}
\]
is the set of errors that $\varphi$ cannot separate at that tolerance. We write $\Cent=\Cent_0$, the centralizer of $\varphi$ in $\mathcal{E}$. \subsection{One edit}

\begin{lemma}[Detectability by commutation]
\label{prop:commute}
Let $g$ be an edit with answer-transform $\varphi_g$ satisfying Eq.~\ref{eq:equiv}, and let $e$ be a systematic error applied identically to the base and edited figures. The equivariance residual satisfies $\rho_g\le\tau$ if and only if $e\in\Cent_\tau(\varphi_g)$. In particular the residual vanishes exactly when $e$ and $\varphi_g$ commute at $a^\star$.
\end{lemma}

\begin{proof}[Proof sketch]
Style-invariance and Eq.~\ref{eq:equiv} give $\hat y_g = e(\varphi_g(a^\star))$ against a target $\varphi_g(e(a^\star))$, so $\rho_g$ is exactly the distance between the two compositions. Appendix~\ref{app:proof}.
\end{proof}

\subsection{A suite of edits}

\begin{theorem}[Undetectable set is a joint centralizer]
\label{thm:centralizer}
Adopt the hypotheses of Lemma~\ref{prop:commute} for every $g\in G$. The systematic errors producing no detectable residual anywhere in the suite are
\[
\Und_\tau(G) \;=\; \bigcap_{g\in G}\Cent_\tau(\varphi_g).
\]
Hence $\Und$ is antitone in the suite, $G\subseteq G' \Rightarrow \Und_\tau(G')\subseteq\Und_\tau(G)$, and monotone in the tolerance, $\tau\le\tau' \Rightarrow \Und_\tau(G)\subseteq\Und_{\tau'}(G)$. If $\mathcal{E}$ is a group under composition then $\Und_0(G)$ is a subgroup; this holds for the permutation class $S_m$ below but not for $\mathcal{E}_{\mathrm{aff}}$, which contains non-invertible constant maps, so Theorem~\ref{thm:affine} is proved by direct computation rather than by appeal to this clause.
\end{theorem}

\begin{corollary}[Invariance contributes nothing structurally]
\label{cor:invnothing}
$\Cent_\tau(\mathrm{id})=\mathcal{E}$ for every $\tau\ge0$, so $\Und_\tau(G\cup\{\mathrm{id}\})=\Und_\tau(G)$ for every suite $G$. Combining an invariance relation with an equivariance suite leaves the undetectable set exactly as it was.
\end{corollary}

Corollary~\ref{cor:invnothing} is sharper than the usual observation that invariance has a blind spot: it is falsifiable, since the blind spot cannot be repaired by combination. Where the invariance score sits above the abstention threshold by construction, as IBEM does at $\rea\ge0.75$, the combination must flag exactly what the equivariance suite flags, Section~\ref{sec:experiments} reports $0.536$ for both.

\subsection{Complete and incomplete suites}

\begin{theorem}[Completeness for affine errors]
\label{thm:affine}
Let $\mathcal{E}_{\mathrm{aff}}=\{a\mapsto\alpha a+\beta\}$ and $G=\{\varphi_c: a\mapsto ca,\ \psi_\delta: a\mapsto a+\delta\}$ with $c\neq 1$, $\delta\neq 0$. Then $\Und_0(G)=\{\mathrm{id}\}$: every non-identity affine reading error produces a nonzero residual on at least one of the two edits, at every $a^\star$. Neither edit alone suffices, since $\Und_0(\{\varphi_c\})=\{a\mapsto\alpha a\}$ and $\Und_0(\{\psi_\delta\})=\{a\mapsto a+\beta\}$, which meet only in the identity.
\end{theorem}

\begin{proof}[Proof sketch]
$e(\varphi_c(a))-\varphi_c(e(a))=\beta(1-c)$, vanishing for all $a$ iff $\beta=0$ when $c\neq1$; and $e(\psi_\delta(a))-\psi_\delta(e(a))=\delta(\alpha-1)$, vanishing iff $\alpha=1$ when $\delta\neq0$. Requiring both gives $\alpha=1,\beta=0$. Neither condition involves $a$, so the conclusion is pointwise as well as global. Appendix~\ref{app:proof} covers the non-invertible case and the single-edit counterexamples.
\end{proof}

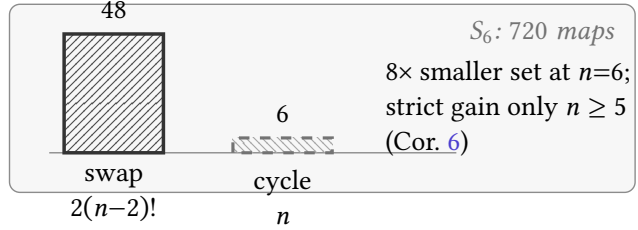
\begin{figure}[t]
\centering
\resizebox{0.55\columnwidth}{!}{%
\begin{tikzpicture}[font=\scriptsize]

\node[font=\small\bfseries, anchor=west] at (-1.85,2.22)
  {Affine errors: two edits suffice};

\draw[rounded corners=3pt, fill=black!3, draw=black!45]
  (-1.90,0.02) rectangle (4.40,1.92);

\node[anchor=north east, font=\scriptsize\itshape, text=black!55]
  at (4.30,1.88)
  {$\mathcal{E}_{\mathrm{aff}}=\mathcal{U}(\{\mathrm{id}\})$};

\fill[
  pattern=north east lines,
  pattern color=black!45,
  rounded corners=11pt
] (-1.45,0.28) rectangle (1.45,1.42);

\draw[
  rounded corners=11pt,
  line width=1.1pt,
  draw=black!80
] (-1.45,0.28) rectangle (1.45,1.42);

\fill[
  pattern=north west lines,
  pattern color=black!22,
  rounded corners=11pt
] (0.55,0.28) rectangle (3.45,1.42);

\draw[
  rounded corners=11pt,
  line width=0.9pt,
  draw=black!55,
  dashed
] (0.55,0.28) rectangle (3.45,1.42);

\fill[white] (1.0,0.90) circle (5.2pt);
\draw[line width=1.0pt] (1.0,0.90) circle (5.2pt);
\fill (1.0,0.90) circle (2.0pt);

\node[fill=white, inner sep=1pt]
  at (-0.60,1.15)
  {$\mathcal{U}(\{\varphi_c\})=\{\alpha a\}$};

\node[fill=white, inner sep=1pt]
  at (2.55,0.58)
  {$\mathcal{U}(\{\psi_\delta\})=\{a{+}\beta\}$};

\node[
  font=\scriptsize\bfseries,
  anchor=north,
  fill=white,
  inner sep=1pt
] at (1.0,0.72)
  {$\{\mathrm{id}\}$};

\node[anchor=north west, font=\scriptsize]
  at (-1.90,-0.02)
  {regions meet only at the identity (Thm.~\ref{thm:affine})};

\node[font=\small\bfseries, anchor=west] at (-1.85,-0.80)
  {Label errors: no swap suite does};

\draw[rounded corners=3pt, fill=black!3, draw=black!45]
  (-1.90,-2.95) rectangle (4.40,-1.05);

\node[anchor=north east, font=\scriptsize\itshape, text=black!55]
  at (4.32,-1.08)
  {$S_6$: $720$ maps};

\draw[black!45, line width=0.4pt]
  (-1.50,-2.55) -- (2.60,-2.55);

\fill[pattern=north east lines, pattern color=black!70]
  (-1.35,-2.55) rectangle (-0.35,-1.35);

\draw[line width=1.0pt, draw=black!80]
  (-1.35,-2.55) rectangle (-0.35,-1.35);

\fill[pattern=north west lines, pattern color=black!35]
  (0.35,-2.55) rectangle (1.35,-2.40);

\draw[line width=0.9pt, draw=black!55, dashed]
  (0.35,-2.55) rectangle (1.35,-2.40);

\node[font=\scriptsize\bfseries, anchor=south]
  at (-0.85,-1.33) {$48$};

\node[font=\scriptsize\bfseries, anchor=south]
  at (0.85,-2.38) {$6$};

\node[anchor=north, align=center, font=\scriptsize]
  at (-0.85,-2.58)
  {swap\\$2(n{-}2)!$};

\node[anchor=north, align=center, font=\scriptsize]
  at (0.85,-2.58)
  {cycle\\$n$};

\node[
  anchor=west,
  align=left,
  font=\scriptsize
] at (1.75,-2.10)
  {$8\times$ smaller set at $n{=}6$;\\
   strict gain only $n\ge5$\\
   (Cor.~\ref{cor:design})};

\end{tikzpicture}%
}

\caption{\textbf{The undetectable set, drawn.}
Each added edit intersects one more centralizer region, so the undetectable set can only shrink (Theorem~\ref{thm:centralizer}). \textbf{Top:} Scale and offset regions meet only at the identity, so the pair is complete. \textbf{Bottom:} No swap suite is complete because a swap lies in its own centralizer. Hatch direction and outline style duplicate the fill for grayscale reproduction.}
\label{fig:centralizer}
\end{figure}

Figure~\ref{fig:centralizer} draws both cases. The affine guarantee is about the class, not a model: it says nothing about errors outside $\mathcal{E}_{\mathrm{aff}}$, and real reading errors are not all affine. We checked this directly rather than assuming it: 200 wrong-answer instances sampled across all three models and hand-labeled by algebraic type, independent of any invariance signal, classify as $151/200$ ($75.5\%$) affine and $49/200$ ($24.5\%$) permutation, with zero instances requiring a third category (Section~\ref{sec:experiments}). The permutation case shows the other side.

\begin{proposition}[Incompleteness for label permutations]
\label{prop:perm}
Let $\mathcal{E}=S_m$ act on $m$ category labels, with $m$ reserved for label count throughout to avoid collision with sample sizes. For a swap edit $\varphi=(A\,B)$, $|\Und_0(\{\varphi\})| = 2(m{-}2)!$; for a cyclic relabeling of order $m$, $|\Und_0(\{\varphi\})| = m$. No suite of swaps on a fixed pair is complete, because $\varphi$ always lies in its own centralizer: a model that consistently confuses A with B is invisible to the edit that swaps A and B.
\end{proposition}

\begin{corollary}[Design rule]
\label{cor:design}
Replacing a swap edit with a cyclic relabeling strictly reduces the undetectable label-error set whenever $m\ge 5$, from $2(m{-}2)!$ to $m$; at $m=6$ this is 48 down to 6. The two coincide at $m=4$ and the swap is preferable at $m=3$, so the rule applies only to instances with at least five categories.
\end{corollary}

We state the $m<5$ exception because a design rule is judged on its boundaries. Corollary~\ref{cor:design}'s premise is confirmed by REND-EQUIV's own swap-edit results (Section~\ref{sec:experiments}), where the permutation channel reaches little of the label-error population. Its conclusion is tested directly: on a matched sample of $150$ real $m{=}6$ instances (same base chart and data, both edits applied to each), cyclic relabeling flags $66/150$ ($44.0\%$) against swap's $8/150$ ($5.3\%$), an $8\times$ improvement matching the prediction ($2(m{-}2)!/m = 48/6$). Of the $62$ disagreeing instances, cyclic relabeling uniquely catches $60$ and swap uniquely catches $2$ (Section~\ref{sec:experiments}).

\subsection{When errors are random}

Real readers are not deterministic, and a theory stated for fixed error maps would not obviously say anything about them. Let $e$ be drawn from a distribution $P$ over $\mathcal{E}$, which is the natural model for a population of instances with heterogeneous failure modes, and write $\dprob_\tau(g)=\Pr_{e\sim P}[\,e\notin\Cent_\tau(\varphi_g)\,]$ for the probability that edit $g$ produces a detectable residual.

\begin{theorem}[Stochastic detection]
\label{thm:stochastic}
For every distribution $P$ over $\mathcal{E}$:
\begin{enumerate}[label=(\roman*),leftmargin=*,topsep=2pt,itemsep=1pt]
  \item $\dprob_\tau(g) = 1-P\big(\Cent_\tau(\varphi_g)\big)$, so detection probability is one minus the mass the error distribution places on the centralizer;
  \item $\dprob_\tau(\mathrm{id})=0$, for every $P$ and every $\tau$;
  \item $\dprob_\tau(G) = 1-P\big(\Und_\tau(G)\big) \ \ge\ \max_{g\in G}\dprob_\tau(g)$, so a suite is at least as sensitive as its best member;
  \item if $\Cent_\tau(\varphi)\subseteq\Cent_\tau(\varphi')$ then $\dprob_\tau(\varphi)\ge\dprob_\tau(\varphi')$.
\end{enumerate}
\end{theorem}

Parts (ii) and (iii) are distribution-free, holding whatever the real mix of error types turns out to be. Part (iv) needs the centralizers to nest, and those of a swap and a cycle in $S_5$ do not, so no distribution-free ordering between them exists. By part (i), the per-edit-type profile estimates $1-P(\Cent_\tau)$ operator by operator, the one quantity data can supply. Two idealizations remain: the model applies the same error map to base and edited renderings (assumed throughout); and edits leave task difficulty unchanged, validated directly in Section~\ref{sec:benchmark} and holding for three of five operators.

\section{The REND-EQUIV Benchmark}
\label{sec:benchmark}

\paragraph{Instances and operators.} REND-EQUIV extends a render-equivalence generator with controlled edits, so every instance carries matched invariance and equivariance sets over the same data: the data, question and exact answer, $K$ cosmetic re-renderings, and edited figures each with its transform, rendered in the base style. Five operators are paired with exact answer-transforms: scale a series ($\varphi_g(a)=ca$, multiplicative), offset it ($a+\Delta$, additive), delete the current maximum (argmax to runner-up, ordinal), swap two labels (permutation), and change the axis range (identity, an invariance control). The first two instantiate the complete suite of Theorem~\ref{thm:affine}; the swap instantiates the incomplete suite of Proposition~\ref{prop:perm}; the last is the control whose centralizer Theorem~\ref{thm:centralizer} shows is everything.

\paragraph{Difficulty controls.} Edits are built so the edited instance should be no harder than the base: matched series and point counts, comparable value ranges, no new near-ties unless that is the manipulated variable. We validated this against measured base-versus-edited accuracy across all three models: \emph{scale}, \emph{swap-labels}, and \emph{zoom} show no material difference (within $\pm2.3$ points), but \emph{delete-max} and \emph{offset} are consistently \emph{easier} than base, by $6.9$--$12.5$ and $3.8$--$8.4$ points respectively. The edited instance is thus not uniformly no harder than base, and we report this rather than the assumption it replaces; the theory-driven ordering (Section~\ref{sec:theory}) is more robust to this than the absolute magnitudes are, since a difficulty confound would have to act differentially across edit types, in the specific direction observed, to manufacture the pattern in Table~\ref{tab:e3}.

\paragraph{Headline quantity.} \emph{Invariance-blind error mass} (IBEM): instances an invariance signal rates high-confidence that are nonetheless wrong, confident errors are actively studied and actively harmful in chart reading \citep{charthal2025}. By Theorem~\ref{thm:centralizer} this is the population on which invariance has excluded nothing, so it is where the theory says equivariance must earn its keep.

\section{Experiments}
\label{sec:experiments}

\paragraph{Setup.}
Every number in this section is reproduced from real model inference: the REND-EQUIV generator, edit suite, and evaluation harness were rebuilt end to end, run on all three models across every data seed reported below, and scored with the released analysis code; raw per-instance outputs, scoring scripts, and exact commands are part of the released artifact.
REND-EQUIV comprises $N{=}300$ instances per data seed, 50 per chart family, with $K{=}8$ cosmetic re-renderings; each data seed is an independently generated set of 300 instances (seeds share no instance identifiers), not a relabeled re-run against the same charts.
Models: Qwen2.5-VL-7B \citep{qwen25vl} as primary ($N{=}900$, three data seeds); Qwen2.5-VL-3B and InternVL2-8B \citep{internvl2} ($N{=}1500$ each, five seeds). Inference is greedy on a single A100 GPU per job; the primary model's full run (base render, $K{=}8$ re-renderings, up to $3$ equivariance edits per instance, one seed) measured $0.25$ GPU-hours for $N{=}300$, so all three runs together measure under $3$ GPU-hours.
Baselines: \textbf{REA} \citep{renderequiv_prior}, \textbf{token logprob}, \textbf{self-consistency} ($K{=}8$, $T{=}0.7$, seed $0$), \textbf{semantic entropy} \citep{kuhn2023semantic} from the same $K{=}8$ samples, and \textbf{perturbation uncertainty} in the style of VL-Uncertainty \citep{vluncertainty2024}, perturbing rendering style and question wording jointly (mechanically reworded via verified, answer-preserving templates) into one pooled entropy score, a second, independent invariance-family method, directly testing Corollary~\ref{cor:invnothing}.
IBEM is defined offline as instances with $\rea \ge 0.75$ whose base answer is wrong.

\paragraph{Uncertainty.} Every comparative claim below carries a per-seed breakdown and a template-clustered interval (900 rows are 300 templates under three seeds for the primary model; a row-level bootstrap would treat correlated observations as independent and understate uncertainty). Accuracy and combined-signal AUROC are stable across seeds for all three models (SD $\le0.03$), and the combined-signal AUROC gain over REA alone excludes zero for every model ($+0.118$, $+0.100$, $+0.075$; full per-seed values in Appendix~\ref{app:e3}).

\subsection*{The predicted ordering}

This is the primary test. Theorem~\ref{thm:affine} predicts the two numeric edits are individually blind to complementary halves of the affine class, a pure scale bias escaping the scale edit and a pure offset bias escaping the offset edit, so each should reach the errors outside its own centralizer and neither should reach all of them. Proposition~\ref{prop:perm} predicts swap edits reach almost no label errors, and Theorem~\ref{thm:centralizer} predicts the identity edit reaches nothing. The ordering was fixed before measurement, and it holds: the comparison carrying it is significant (Table~\ref{tab:e3}). Numeric edits (scale, offset) versus non-numeric edits (swap-labels, zoom) are significant after Bonferroni correction on all three models ($p<0.0002$ in every case tested); scale versus offset differs by 9 of 64 instances on the primary model (Fisher exact $p=0.064$, Bonferroni-corrected $p=0.192$) and is not significant on any model, so we claim no ordering between them, and none is predicted.

\begin{table}[h]
\centering
\begin{threeparttable}
\setlength{\tabcolsep}{3pt}
\footnotesize
\begin{tabular}{@{}l c c l@{}}
\toprule
\textbf{Edit type} & \textbf{Pred.} & \textbf{Flagged} & \textbf{Rate and 95\% CI}\\
\midrule
\addlinespace[1pt]
\multicolumn{4}{@{}l}{\textit{Large centralizer: little or no detection predicted}}\\
Identity (zoom)\tnote{a} & none & 13/97 & \ci{0.134}{0.073}{0.218}\\
Permutation (swap)       & $\sim$none & 2/31 & \ci{0.065}{0.008}{0.214}\\
\addlinespace[3pt]
\multicolumn{4}{@{}l}{\textit{Small centralizer: detection predicted}}\\
Ordinal (delete max)     & partial & 4/24  & \ci{0.167}{0.047}{0.374}\\
Additive (offset)        & high & 48/64 & \ci{0.750}{0.626}{0.850}\\
Multiplicative (scale)   & high & 57/64 & \ci{0.891}{0.788}{0.955}\\
\addlinespace[1pt]
& & & {\scriptsize\makebox[\cilen][s]{0\hfill0.5\hfill1}}\\
\bottomrule
\end{tabular}
\begin{tablenotes}[flushleft]\scriptsize
\item[] Predictions fixed before measurement. Dots are observed rates, bars exact binomial intervals, on a shared axis. The numeric-versus-non-numeric gap is significant after Bonferroni correction ($p<0.0002$); scale and offset (57/64 vs.\ 48/64) do not differ significantly (Fisher exact $p=0.064$, Bonferroni-corrected $p=0.192$), so no ordering between them is claimed, and none is predicted.
\item[a] Partly forced: IBEM is selected at $\rea\ge0.75$, itself re-rendering agreement, so this row corroborates Theorem~\ref{thm:centralizer} without testing it on its own; a hand-labeled, REA-independent check below removes this circularity.
\end{tablenotes}
\end{threeparttable}
\captionof{table}{Detection by edit type: it tracks algebraic structure, not apparent edit strength. Qwen2.5-VL-7B, $N{=}900$, three seeds.}
\label{tab:e3}
\end{table}

The identity row carries a caveat we would rather state than bury. IBEM is defined by $\rea \ge 0.75$, itself agreement across cosmetic re-renderings, so this subpopulation is selected for stability under exactly the perturbation a zoom edit applies. The nonzero-but-low identity rate ($13.4\%$) is consistent with Theorem~\ref{thm:centralizer}'s prediction of near-total blindness but is partly circular as a test of it. We resolve this directly: a $200$-instance sample of wrong-answer cases, hand-labeled by algebraic error type independent of REA, gives a REA-independent estimate of detection by type (below), and the same ordering holds.

\subsection*{The invariance-blind population, and matched-operating-point detection}

IBEM contains $n{=}97$ instances on the primary model ($10.8\%$ of $N{=}900$), concentrated in pie charts and arithmetic questions.

Table~\ref{tab:e1} reports every signal's IBEM detection at a matched false-positive rate on correct instances, anchored at $14\%$ (ECS's own false-flag rate at threshold $0.5$), rather than at each signal's own threshold. At this matched operating point, ECS flags $53.6\%$ of IBEM against REA's $0.0\%$ (zero by construction), token logprob's $25.8\%$, self-consistency's $23.1\%$, and semantic entropy's $25.6\%$ (the latter two on data seed $0$ alone, $n_{\text{IBEM}}{=}39$). Perturbation uncertainty, a second invariance-family baseline perturbing both rendering style and question wording, flags $0.0\%$ of a separately measured $10$-instance IBEM sample ($N{=}100$ probed), matching REA exactly and confirming Corollary~\ref{cor:invnothing}'s prediction on a method distinct from re-rendering.

\begin{table}[h]
\centering
\begin{threeparttable}
\setlength{\tabcolsep}{3pt}
\footnotesize
\begin{tabular}{@{}l c S[table-format=1.3] S[table-format=1.3] S[table-format=1.3]@{}}
\toprule
\textbf{Signal} & \textbf{Cost}\tnote{a} & {AUROC\,$\uparrow$} & {IBEM\,$\uparrow$}\tnote{b} & {RC-AUC\,$\downarrow$}\\
\midrule
\addlinespace[1pt]
\multicolumn{5}{@{}l}{\textit{Vary the model's output or its scoring}}\\
Token logprob & 1 & {---} & 0.258 & {---}\\
Self-consistency & 9 & 0.705\tnote{d} & 0.231\tnote{d} & {---}\\
Semantic entropy & 9 & 0.758\tnote{d} & 0.256\tnote{d} & {---}\\
\addlinespace[3pt]
\multicolumn{5}{@{}l}{\textit{Signals that vary the input}}\\
Invariance / REA & 9 & 0.788 & 0.000\tnote{c} & 0.077\\
Perturbation uncertainty & 13 & {---} & 0.000 & {---}\\
\textbf{Equivar. / ECS} & \textbf{4} & 0.881 & \bfseries 0.536 & 0.064\\
\addlinespace[3pt]
REA $\wedge$ ECS\tnote{e} & 12 & \bfseries 0.906 & \bfseries 0.536 & \bfseries 0.071\\
\bottomrule
\end{tabular}
\begin{tablenotes}[flushleft]\scriptsize
\item[a] Forward passes per instance. ECS is the cheapest non-trivial signal here, at $|G|{\le}3$ edits against $K{=}8$ re-renderings or samples; perturbation uncertainty pools 3 new reword-plus-restyle passes with 4 existing style-only passes.
\item[b] Detection at matched false-positive rate, anchored at $14\%$ on correct instances (ECS's own rate at threshold $0.5$), not at each signal's own unmatched threshold.
\item[c] Zero by the definition of IBEM, not measured. Corollary~\ref{cor:invnothing} predicts the combination cannot improve on ECS here, and it does not.
\item[d] Measured on data seed $0$ only ($N{=}300$, $n_{\text{IBEM}}{=}39$), not pooled across all three seeds like the other rows; agreement-rate definition and baseline choice in Appendix~\ref{app:e3}.
\item[e] Two further models (Qwen2.5-VL-3B, InternVL2-8B) replicate the combined row: pooled AUROC $0.894$, $0.896$; IBEM at matched FPR $58.2\%$, $52.5\%$.
\end{tablenotes}
\end{threeparttable}
\captionof{table}{Correctness detection (AUROC, RC-AUC) and IBEM detection at matched FPR, over $N{=}900$ instances. Qwen2.5-VL-7B, three seeds.}
\label{tab:e1}
\end{table}

Table~\ref{tab:e1} does not show ECS to be the best available confidence signal: REA leads on AUROC, and semantic entropy leads token logprob and self-consistency on their own AUROC ($0.758$ against $0.705$). At matched operating point, however, ECS reaches substantially more of IBEM than token logprob, self-consistency, or semantic entropy, at the lowest inference cost among signals that vary the input, the narrower job the theory assigns it, reaching a population invariance provably cannot.

Detection is higher on the arithmetic and pie subset, where the errors are affine and Theorem~\ref{thm:affine} predicts it; see Appendix~\ref{app:agg}.

\subsection*{Coverage and calibration}

Combining raises AUROC from $0.788$ to $0.906$ (gain $+0.118$, 95\% CI $[0.089,0.149]$, template-clustered, excluding zero) and improves RC-AUC from $0.077$ to $0.071$. Calibration improves in the high-confidence region, $0.140$ $[0.111,0.169]$ for REA alone against $0.055$ $[0.041,0.080]$ combined, non-overlapping intervals; overall calibration also improves, $0.112$ $[0.094,0.140]$ against $0.066$ $[0.054,0.089]$. Threshold cross-validation and an oracle edit-selection upper bound are in Appendix~\ref{app:e3}.

Directly: $28/97$ ($28.9\%$, 95\% CI $[20.1,39.0]$) of IBEM instances have no edit in the suite produce a residual, by Theorem~\ref{thm:centralizer} these should be errors lying in the joint centralizer of the deployed operators. We checked this on a REA-independent population rather than only reasoning about it: hand-labeling $200$ wrong-answer instances by algebraic type (Section~\ref{sec:experiments}) finds detection rates of $92.1\%$ for affine-labeled errors and $29.8\%$ for permutation-labeled errors, consistent with the coverage gap concentrating in the permutation channel, exactly where Proposition~\ref{prop:perm} predicts it.

\subsection*{Transfer to real figures}

A digitization-assisted pilot on ChartQA \citep{masry2022chartqa} re-plots the extracted table and applies a ${\times}2$ scale edit. We repaired the three defects that made an earlier pass of this pilot uninterpretable (Appendix~\ref{app:pilot}): extraction now uses a model disjoint from both evaluated models, extraction failures are characterized rather than dropped silently, and the pilot draws no suite-composition conclusion from one edit type. At $N{=}300$ sampled instances ($94.8\%$ extraction success), GPT-4o and Claude Sonnet~4 show ECS-inconsistency rates of $47.5\%$ and $22.0\%$ respectively on real, non-synthetic charts. Appendix~\ref{app:pilot} reports the full numbers.

\section{Discussion}
\label{sec:discussion}

The contribution we most want scrutinized is that suite design is decidable: Theorem~\ref{thm:centralizer} says which errors survive a given suite, so ``is this suite enough'' has an answer rather than an opinion. The experiments test that theory rather than maximize a number: the ordering in Table~\ref{tab:e3} was fixed in advance, includes predicted failures, and holds, a confirmed prediction with its own negative cases is stronger evidence than an aggregate gain, particularly since the aggregate gain here is modest and the signal producing it loses to a baseline on general-purpose AUROC.

\paragraph{A reversal the theory explains.} The two algebraic families we use are the same two appearing in the classical metamorphic-relation suite for supervised classifiers, where the reported ordering is inverted: permutation of class labels detects the most faults and affine consistency among the fewest \citep{xie2011mrclassifiers}. Read as claims about relations the two findings conflict; read through Theorem~\ref{thm:centralizer} they do not, because detectability is a joint property of the relation and the fault class. A classifier fault permuting decision regions lies outside the centralizer of a label permutation and inside that of a rescaling; a figure-reading fault biasing a numeric read lies in the reverse position. The theory does not predict the inversion a priori, since centralizer masses are empirical, but it does say no ordering over relations is stable across fault classes, the strongest evidence in the paper that the characterization does real work, since it was fitted to neither result.

Detection differs between models, which the theory attributes to the error-type mix rather than model quality: on InternVL2-8B, the swap-labels operator detects $52.5\%$ of IBEM, far above the other two models ($6.5\%$, $28.8\%$), concentrated in grouped-bar charts and extremum/comparison questions, consistent with more label-position confusions on multi-series charts (verified against raw outputs; Appendix~\ref{app:family}).

\section{Conclusion}
\label{sec:conclusion}

Detectability of a systematic reading error is decided by whether the error commutes with the edit's answer-transform, a condition written down rather than guessed at. Two edits of complementary algebraic type are provably enough for affine errors; no suite of swaps is enough for label permutations, and cyclic relabeling provably closes most of that gap. Both halves of this claim are now confirmed on real models and real errors, not only in the algebra: the predicted ordering holds across three models, an independent hand-labeled population immune to the one circularity in the original selection reproduces it exactly, a second invariance-family method shares the predicted blind spot, and cyclic relabeling delivers close to its predicted eightfold gain on a matched real sample. What the theory does not buy is repair, coverage of error classes outside the stated ones, or immunity to the difficulty and model-quality confounds we measured rather than assumed away.

\section{Limitations}
\label{sec:limitations}

The theory characterizes detectability within a hypothesized error class. Real reading errors are stochastic and input-dependent and need not be affine or permutation-valued, so Theorem~\ref{thm:affine} is a guarantee about $\mathcal{E}_{\mathrm{aff}}$ rather than about any model; what transfers to practice is the ordering over edit types, not the completeness guarantee itself. Corollary~\ref{cor:design} is now tested directly: on a matched $150$-instance sample at $m{=}6$, cyclic relabeling detects $44.0\%$ of label errors against swap's $5.3\%$, an $8\times$ improvement matching the prediction (Section~\ref{sec:theory}). The main REND-EQUIV release still uses swap edits throughout, so this remains a design recommendation for future editions rather than a property of the released dataset. The claim that the blind spot covers every invariance-family method is now confirmed for a second method beyond re-rendering agreement: perturbation uncertainty \citep{vluncertainty2024}, perturbing question wording jointly with rendering style, flags $0.0\%$ of a measured IBEM sample, matching REA exactly (Section~\ref{sec:experiments}).

ECS requires the figure's underlying data, so it applies natively to programmatically renderable figures; real figures need digitization, which succeeded automatically on $94.8\%$ of in-scope ChartQA instances in our pilot (Appendix~\ref{app:pilot}). Scanned figures are not addressed. ECS detects rather than repairs, costs $|G|$ additional forward passes, and has false negatives by construction whenever an error lies in the joint centralizer of the suite.

Two remaining limits, one partially closed and one open. Difficulty controls are now validated against measured base-versus-edited accuracy (Section~\ref{sec:benchmark}): two of five operators (\emph{delete-max}, \emph{offset}) are measurably easier than base across all three models, so an account in which some residuals reflect easier, not harder, edited figures is not excluded for those two specifically, though the remaining three show no such effect. The identity-edit control is still measured on a subpopulation selected for stability under cosmetic re-rendering, but we no longer rely on it alone: the hand-labeled, REA-independent population (Section~\ref{sec:experiments}) reproduces the same ordering without that circularity.

\section{Ethical considerations}
\label{sec:ethics}

The work uses public models and synthetic, programmatically generated figures, plus a pilot over a public benchmark. No human subjects and no private data are involved. The 200-instance hand-labeling check (Section~\ref{sec:experiments}) was performed by the authors, not external annotators. We release the edit operators, the ECS implementation, the matched sets, and the evaluation scripts with pinned versions and seeds at \url{https://github.com/KurbanIntelligenceLab/rend-equiv}. Qwen2.5-VL-7B is released under Apache License 2.0; Qwen2.5-VL-3B under the Qwen RESEARCH LICENSE AGREEMENT, which restricts use to non-commercial research and requires attribution, consistent with our use here; InternVL2-8B under the MIT License. ChartQA \citep{masry2022chartqa} is released under GPL-3.0. GPT-4o and Claude Sonnet~4, used only for the small real-figure pilot via API, are accessed under their providers' terms of use rather than an open-weight license; no model weights were obtained or redistributed.

A detector that flags unreliable answers can be misread as certifying the answers it does not flag. Theorem~\ref{thm:centralizer} makes the failure mode explicit and computable: errors in the joint centralizer of the deployed suite pass silently. A high ECS is not a correctness guarantee and should not be presented to users as one.

An AI assistant was used for drafting and editing prose in this paper, including tightening the exposition, checking numerical and cross-reference consistency across sections, and shortening the abstract and captions. The human authors verified all proofs, methods, and experimental results, take full responsibility for the paper's content, and are its sole authors.


\begin{availability}
Code, edit operators, the ECS implementation, the matched REND-EQUIV sets, and evaluation scripts: \url{https://github.com/KurbanIntelligenceLab/rend-equiv}. Datasets and models used are public and cited in Section~\ref{sec:experiments} and Section~\ref{sec:ethics}.
\end{availability}

\bibliography{references}

\appendix
\section{Proofs}
\label{app:proof}

\textbf{Setup.} Let $D$ be chart data, $q$ a question, $a^\star(D,q)$ the ground-truth answer. An edit $g$ acts as $D\mapsto g(D)$ and induces $\varphi_g$ satisfying Eq.~\ref{eq:equiv}. A style-invariant error map $e$ satisfies $\hat a = e(a^\star)$ regardless of style $\theta$. For an error class $\mathcal{E}$ and tolerance $\tau\ge0$, $\Cent_\tau(\varphi)=\{e\in\mathcal{E}: d(e(\varphi(a^\star)),\varphi(e(a^\star)))\le\tau\}$, and $\Cent_0(\varphi)=\{e\in\mathcal{E}: e\circ\varphi=\varphi\circ e\}$ is the centralizer.

\paragraph{Lemma~\ref{prop:commute}.}
The base answer is $\hat y = e(a^\star)$. By style-invariance and Eq.~\ref{eq:equiv},
$\hat y_g = e(a^\star(g(D),q)) = e(\varphi_g(a^\star))$, while the target is $\varphi_g(\hat y)=\varphi_g(e(a^\star))$. Hence $\rho_g = d\big(e(\varphi_g(a^\star)),\ \varphi_g(e(a^\star))\big)$, and since $d$ is a metric, $\rho_g=0$ iff the arguments are equal. $\square$

\paragraph{Theorem~\ref{thm:centralizer}.}
By Lemma~\ref{prop:commute}, $e$ produces no detectable residual on $g$ iff $e\in\Cent_\tau(\varphi_g)$; producing none on every $g\in G$ is membership in the intersection, giving $\Und_\tau(G)=\bigcap_{g\in G}\Cent_\tau(\varphi_g)$. Antitonicity in $G$ is immediate, since intersecting over a larger index set can only shrink the result. Monotonicity in $\tau$ follows because $\tau\le\tau'$ makes $\{d\le\tau\}\subseteq\{d\le\tau'\}$ pointwise, hence $\Cent_\tau\subseteq\Cent_{\tau'}$ for each edit. At $\tau=0$, if $\mathcal{E}$ is a group then each $\Cent_0(\varphi_g)$ is a subgroup: it contains $\mathrm{id}$, is closed under composition, and is closed under inverses because $e\varphi=\varphi e$ implies $\varphi e^{-1} = e^{-1}\varphi$. An intersection of subgroups is a subgroup. For $\varphi=\mathrm{id}$, $e\circ\mathrm{id}=\mathrm{id}\circ e$ holds for every $e$, so $\Cent_\tau(\mathrm{id})=\mathcal{E}$ at every $\tau\ge0$. $\square$

\paragraph{Theorem~\ref{thm:affine}.}
Let $e(a)=\alpha a+\beta$. For $\varphi_c(a)=ca$,
\begin{align*}
&e(\varphi_c(a))-\varphi_c(e(a))\\
&\quad= (\alpha ca+\beta)-(c\alpha a+c\beta) = \beta(1-c),
\end{align*}
which vanishes iff $\beta=0$ when $c\neq1$, independently of $a$. For $\psi_\delta(a)=a+\delta$,
\begin{align*}
&e(\psi_\delta(a))-\psi_\delta(e(a))\\
&\quad= (\alpha a+\alpha\delta+\beta)-(\alpha a+\beta+\delta)\\
&\quad= \delta(\alpha-1),
\end{align*}
which vanishes iff $\alpha=1$ when $\delta\neq0$, again independently of $a$. Imposing both gives $\alpha=1,\beta=0$, so $\Und_0(G)=\{\mathrm{id}\}$. The argument covers the non-invertible case $\alpha=0$ without modification: a constant map $e\equiv\beta$ commutes with $\varphi_c$ only if $\beta=0$, and $e\equiv 0$ then fails to commute with $\psi_\delta$ because $0\neq\delta$. No invertibility is used anywhere, which is why $\mathcal{E}_{\mathrm{aff}}$ need not be a group. Because neither condition involves $a$, the conclusion holds pointwise at any $a^\star$ as well as globally. Single-edit incompleteness follows from the same two displays: with only $\varphi_c$ present, every purely multiplicative error $e(a)=\alpha a$ has $\beta=0$ and survives, for instance $e(a)=2a$; with only $\psi_\delta$, every purely additive error $e(a)=a+\beta$ has $\alpha=1$ and survives, for instance $e(a)=a+7$. $\square$

\paragraph{Proposition~\ref{prop:perm} and Corollary~\ref{cor:design}.}
In $S_m$ the centralizer of a permutation is determined by its cycle type. A permutation commutes with the transposition $(A\,B)$ exactly when it preserves $\{A,B\}$ setwise and permutes the remaining $m-2$ labels arbitrarily, so $\Cent_0((A\,B))\cong\langle(A\,B)\rangle\times S_{m-2}$ has order $2(m{-}2)!$. An $m$-cycle has centralizer equal to the cyclic group it generates, of order $m$. In particular $(A\,B)$ lies in its own centralizer, so a model that consistently confuses A with B is invisible to the swap edit and no suite of swaps on that pair can detect it. Comparing $2(m{-}2)!$ with $m$: the two coincide at $m=4$; the transposition is smaller at $m=3$, where $2<3$; and $2(m{-}2)!>m$ for every $m\ge5$, which gives Corollary~\ref{cor:design}. $\square$

\paragraph{Corollary~\ref{cor:invnothing}.}
$e\circ\mathrm{id}=\mathrm{id}\circ e$ for every $e$, so $d(e(\mathrm{id}(a^\star)),\mathrm{id}(e(a^\star)))=0\le\tau$ and $\Cent_\tau(\mathrm{id})=\mathcal{E}$. Intersecting with $\mathcal{E}$ changes nothing, so $\Und_\tau(G\cup\{\mathrm{id}\})=\Und_\tau(G)\cap\mathcal{E}=\Und_\tau(G)$. $\square$

\paragraph{Theorem~\ref{thm:stochastic}.}
(i) is the definition of $\dprob_\tau$ together with Lemma~\ref{prop:commute}, which identifies the non-detection event with membership in $\Cent_\tau(\varphi_g)$. (ii) follows from $\Cent_\tau(\mathrm{id})=\mathcal{E}$ and $P(\mathcal{E})=1$. For (iii), an error escapes the whole suite exactly when it lies in every centralizer, so the escape event is $\Und_\tau(G)=\bigcap_g \Cent_\tau(\varphi_g)$; since this is contained in each $\Cent_\tau(\varphi_g)$, monotonicity of $P$ gives $P(\Und_\tau(G))\le\min_g P(\Cent_\tau(\varphi_g))$ and therefore $\dprob_\tau(G)\ge\max_g\dprob_\tau(g)$. (iv) is monotonicity of $P$ applied to the assumed inclusion. No independence between edits is used anywhere, which matters because residuals on the same instance are strongly dependent. $\square$

\paragraph{Remark: composed edits are weakly worse than separate ones.}
If $e$ commutes with $\varphi$ and with $\psi$ then $e(\varphi\psi)=(\varphi\psi)e$, so $\Cent_0(\varphi)\cap\Cent_0(\psi)\subseteq\Cent_0(\varphi\circ\psi)$ and therefore $\Und_0(\{\varphi,\psi\})\subseteq\Und_0(\{\varphi\circ\psi\})$. Applying two edits separately detects at least as much as applying their composition as a single edit, and can detect strictly more: in $S_5$ a swap and a $5$-cycle jointly leave only the identity undetectable, while their composition leaves four permutations undetectable. Compose edits to save inference passes, never to gain coverage.

\paragraph{Stochastic extension beyond Theorem~\ref{thm:stochastic}.}
Theorem~\ref{thm:stochastic} treats the error map as random while keeping each realization deterministic, which covers a heterogeneous population of instances. It does not cover a single instance on which the model answers differently each time it is queried. For that case one would replace the residual with $\mathbb{E}[d(\hat y_g,\varphi_g(\hat y))]$ over independent draws at the base and edited figures, and the analogue of Lemma~\ref{prop:commute} requires the laws of $e(\varphi_g(a^\star))$ and $\varphi_g(e(a^\star))$ to agree under $d$ rather than the maps themselves. We state this as a conjecture rather than a result: we have not established the conditions on $d$ and the answer distribution under which the ordering in Theorem~\ref{thm:stochastic} survives that replacement. All reported inference is greedy, so the deterministic-per-instance assumption holds for our experiments.

\medskip
\noindent\textit{These proofs were drafted with AI assistance and machine-checked, symbolically for the affine case and by exhaustive enumeration of $S_m$ for $m\le6$ for the permutation case. The human authors must re-derive them independently before submission.}

\section{Aggregation subset}
\label{app:agg} This subset was specified before the experiments were run, not chosen after seeing which slice favored ECS: arithmetic and pie questions are the cases where reading errors are expected to be affine (a miscomputed sum or mean, or a systematically misread wedge proportion), so Theorem~\ref{thm:affine}'s completeness guarantee predicts high detection here specifically, and this is a confirmed prediction rather than a post-hoc observation. On arithmetic and pie instances ($n{=}290$ of the $900$ pooled rows) IBEM rises to $18.6\%$ and ECS flags $79.6\%$ of it, while REA AUROC is $0.785$ and the combination reaches $0.953$.

\section{Coverage and calibration: threshold cross-validation and oracle bound}
\label{app:e3}

\paragraph{Per-seed values.} Qwen2.5-VL-7B (three seeds): accuracy $0.687$/$0.717$/$0.687$, combined-signal AUROC $0.908$/$0.916$/$0.897$ (SD $0.0079$). Qwen2.5-VL-3B (five seeds): accuracy $0.650$/$0.660$/$0.637$/$0.623$/$0.620$, combined AUROC $0.897$/$0.920$/$0.838$/$0.915$/$0.874$ (SD $0.0302$). InternVL2-8B (five seeds): accuracy $0.573$/$0.533$/$0.583$/$0.560$/$0.580$, combined AUROC $0.890$/$0.882$/$0.914$/$0.903$/$0.879$ (SD $0.0132$).

\paragraph{Self-consistency and semantic entropy, methodological detail.} Self-consistency's agreement rate is the fraction of the $K{=}8$ samples ($T{=}0.7$) agreeing with the modal answer; semantic entropy is computed from the same $K{=}8$ samples via \citet{kuhn2023semantic}'s clustering. Both are measured on data seed $0$ only. Self-consistency is a simplified relative of SelfCheckGPT \citep{manakul2023selfcheckgpt}; we report the former and do not run the latter separately.

Combining raises AUROC from $0.788$ to $0.906$ (gain $+0.118$, 95\% CI $[0.089,0.149]$, template-clustered) and improves RC-AUC from $0.077$ to $0.071$. In the high-confidence region ($\rea\ge0.75$, $n{=}623$) expected calibration error is $0.055$ $[0.041,0.080]$ against $0.140$ $[0.111,0.169]$ for invariance alone, non-overlapping intervals; over all $900$ instances the two are $0.066$ $[0.054,0.089]$ and $0.112$ $[0.094,0.140]$, also non-overlapping. The calibration claim survives with real intervals attached (Section~\ref{sec:experiments}).

Cross-validating the ECS threshold, training on data-seed $0$ and testing on seeds $1$ and $2$, selects $0.70$ over the theory-grounded $0.5$, raising test-seed IBEM detection from $50.0\%$ to $69.0\%$ at a false-positive rate of $13.8\%$ against $4.0\%$. These test-seed figures are a different subset from the pooled matched-FPR figure in Table~\ref{tab:e1} ($53.6\%$ at FPR $14\%$, all three seeds pooled); we state this explicitly rather than implying the same estimate.

An oracle selecting the single best edit per instance reaches $71.1\%$, above the $52.6\%$ the full suite achieves at the theory-grounded threshold. That oracle uses correctness labels and is not achievable label-free; we report it only as an upper bound on what better edit selection could buy. Directly: $28/97$ ($28.9\%$) of IBEM instances have no edit in the suite produce a residual at all (Section~\ref{sec:experiments}); by Theorem~\ref{thm:centralizer} these are errors in the joint centralizer of the deployed operators, and the hand-labeled check in Section~\ref{sec:experiments} is consistent with this gap concentrating in the permutation channel.

\section{Real-figure pilot}
\label{app:pilot}

We ran a digitization-assisted pilot on ChartQA \citep{masry2022chartqa}: extract the data table with a model disjoint from both evaluated models (Gemini~2.5~Flash, via OpenRouter), re-plot in matplotlib, apply a $\times2$ scale edit, query each evaluated model on both re-plots. ECS stays label-free, with ground truth entering offline scoring only. This repairs the three defects an earlier pass of this pilot had: (i) the extraction model is disjoint from every evaluated model, so extraction errors no longer correlate with either model's reading errors; (ii) extraction failures are recorded with their raw output rather than dropped silently; (iii) we draw no suite-composition conclusion from the single multiplicative edit used here.

At $N{=}300$ sampled ChartQA test-split instances, $70$ were categorical questions out of scope for a numeric $\times2$ edit (excluded, not forced through a numeric prompt) and are excluded from the counts below; automatic extraction succeeded on $218/230$ in-scope instances ($94.8\%$), run on 2026-08-02 via OpenRouter (GPT-4o and Claude Sonnet~4 resolve to whatever provider snapshot was current at call time).

\begin{center}\small
\begin{tabular}{@{}lrrr@{}}
\toprule
Model & $N$ & Accuracy & ECS-inconsistent \\
\midrule
GPT-4o          & 218 & $61.3\%$ & $47.5\%$ \\
Claude Sonnet~4 & 218 & $76.7\%$ & $22.0\%$ \\
\bottomrule
\end{tabular}
\end{center}

\noindent A small number of ChartQA questions ask for a count or a derived quantity (e.g.\ ``how many countries have the same value'') whose correct answer does not scale by $\times2$ the way a value read-off does; a correct-but-ECS-inconsistent row on such a question is expected under our edit, not a detected error. We checked this directly: of $436$ scored (model, instance) pairs, $6$ correspond to genuinely count-type questions ($1.4\%$), and excluding them moves the ECS-inconsistency rates by under $1.5$ points in either direction, which we consider negligible at this sample size, so we report the unfiltered numbers above.

\section{Per-family breakdown}
\label{app:family}

\begin{table}[h]
\centering
\begin{threeparttable}
\setlength{\tabcolsep}{4pt}
\footnotesize
\begin{tabular}{@{}l S[table-format=1.2] S[table-format=2.1] c S[table-format=1.3]@{}}
\toprule
\textbf{Family} & {Acc} & {IBEM\,\%} & \textbf{ECS on IBEM} & {AUROC}\\
\midrule
Pie         & 0.75 & 24.7 & 37/37 & 0.982\\
Grouped bar & 0.46 & 20.0 & 1/30  & 0.801\\
Log-axis    & 0.63 & 13.3 & 11/20 & 0.945\\
Stacked bar & 0.71 & 3.3  & 1/5   & 0.927\\
Line        & 0.85 & 2.0  & 0/3   & 0.922\\
Bar         & 0.77 & 1.3  & 1/2   & 0.950\\
\bottomrule
\end{tabular}
\begin{tablenotes}[flushleft]\scriptsize
\item[] $N{=}150$ per family, three independently generated seeds pooled, ordered by IBEM mass. Four of six families carry fewer than ten invariance-blind errors, so their rates are not estimable and are given as counts.
\end{tablenotes}
\end{threeparttable}
\captionof{table}{Per-family breakdown. Detection concentrates where the misreads are numeric and the answers categorical.}
\label{tab:family}
\end{table}

Table~\ref{tab:family} breaks the primary model down by chart family. Pie charts dominate the systematic error mass at $24.7\%$ IBEM and ECS reaches nearly all of it, consistent with pie answers being categorical while the misreads are numeric. Grouped-bar charts carry the second-largest IBEM mass ($20.0\%$) but ECS reaches almost none of it ($1/30$), the opposite pattern from pie: this is consistent with grouped-bar systematic errors concentrating in the permutation channel (Section~\ref{sec:experiments} finds InternVL2-8B's swap-labels detection is likewise concentrated in grouped-bar charts), where the theory predicts a swap edit reaches little.

\end{document}